%% file: sample-authordraft.tex
\documentclass[sigconf]{acmart}

\input{math_commands.tex}

\usepackage{hyperref}
\usepackage{url}
\usepackage{xcolor} 
\usepackage{subcaption}
\usepackage{multirow}
\usepackage{wrapfig}
\usepackage{mathtools}
\usepackage{amsthm}
\usepackage{thm-restate}
\usepackage{enumerate}
\usepackage{enumitem}
\usepackage{xspace}
\usepackage{xfrac}
\usepackage{upgreek}
\usepackage[capitalize,noabbrev]{cleveref}
\usepackage{bm}

\newcommand*\diff{\mathop{}\!\mathrm{d}}
\newcommand*\PM{$\text{PM}_{2.5}$\xspace}
\newcommand*\tresnet{{\textsc{tresnet}}\xspace}
\newcommand*\tresnetpl{{\textsc{tresnet}$_\textsc{pl}$}\xspace}
\newcommand*\tresnetvc{{\textsc{tresnet}$_\textsc{vc}$}\xspace}

\newcommand*\mug{$\mathrm{\upmu g/m^3}$\xspace}

\newcommand*{\Rad}{\mathrm{Rad}_n}
\newcommand*{\nn}{^\text{NN}}
\newcommand*{\tr}{^\text{tr}}

\newcommand*{\vpsi}{{\bm \psi}}
\newcommand*{\vvarphi}{{\bm \varphi}}
\newcommand*{\vepsilon}{{\bm \epsilon}}

\newcommand{\indep}{\perp \!\!\!\! \perp}

\newcommand{\rev}[1]{{#1}}

\declaretheorem[name=Proposition]{proposition}
\declaretheorem[name=Assumption,numberwithin=section]{assumption}

\AtBeginDocument{%
  \providecommand\BibTeX{{%
    \normalfont B\kern-0.5em{\scshape i\kern-0.25em b}\kern-0.8em\TeX}}}

\copyrightyear{2024}
\acmYear{2024}
\setcopyright{acmlicensed}\acmConference[KDD '24]{Proceedings of the 30th ACM SIGKDD Conference on Knowledge Discovery and Data Mining}{August 25--29, 2024}{Barcelona, Spain}
\acmBooktitle{Proceedings of the 30th ACM SIGKDD Conference on Knowledge Discovery and Data Mining (KDD '24), August 25--29, 2024, Barcelona, Spain} \acmDOI{10.1145/3637528.3671761}
\acmISBN{979-8-4007-0490-1/24/08}

\begin{document}

\title[Causal Estimation of Exposure Shifts with Neural Networks]{Causal Estimation of Exposure Shifts with Neural Networks}

\author{Mauricio Tec}
\email{mauriciogtec@hsph.harvard.edu}
\orcid{0000-0002-1853-5842}
\affiliation{%
  \institution{Harvard University}
  \city{Cambridge}
  \state{MA}
  \country{USA}
}

\author{Kevin Josey}
\orcid{0000-0003-2490-6272}
\authornote{Work conducted while at Harvard University}
\email{kevin.josey@cuanschutz.edu}
\affiliation{%
  \institution{Colorado School of Public Health}
  \city{Aurora}
  \state{CO}
  \country{USA}
}

\author{Oladimeji Mudele}
\orcid{0000-0001-7131-6334}
\email{omudele@hsph.harvard.edu}
\affiliation{%
  \institution{Harvard University}
  \city{Cambridge}
  \state{MA}
  \country{USA}
}

\author{Francesca Dominici}
\orcid{0000-0002-9382-0141}
\email{fdominic@hsph.harvard.edu}
\affiliation{%
  \institution{Harvard University}
  \city{Cambridge}
  \state{MA}
  \country{USA}
}

\begin{abstract}
A fundamental task in causal inference is estimating the effect of a \emph{distribution} shift in the treatment variable. We refer to this problem as \emph{shift-response function} (SRF) estimation. Existing neural network methods for causal inference lack theoretical guarantees and practical implementations for SRF estimation. In this paper, we introduce \textbf{T}argeted \textbf{R}egularization for \textbf{E}xposure \textbf{S}hifts with Neural \textbf{Net}works (\tresnet), a method to estimate SRFs with robustness and efficiency guarantees. Our contributions are twofold. First, we propose a targeted regularization loss for neural networks with theoretical properties that ensure double robustness and asymptotic efficiency specific to SRF estimation. Second, we extend targeted regularization to support loss functions from the exponential family to accommodate non-continuous outcome distributions (e.g., discrete counts). We conduct benchmark experiments demonstrating \tresnet 's broad applicability and competitiveness. We then apply our method to a key policy question in public health to estimate the causal effect of revising the US National Ambient Air Quality Standards (NAAQS) for $\text{PM}_{2.5}$ from 12 $\mu g/m^3$ to 9 $\mu g/m^3$. This change has been recently proposed by the US Environmental Protection Agency (EPA). Our goal is to estimate the reduction in deaths that would result from this anticipated revision using data consisting of 68 million individuals across the U.S.\footnote{The code is available at \url{https://github.com/NSAPH-Projects/tresnet}}
\end{abstract}

\maketitle

\section{Introduction}

The field of causal inference has seen immense progress in the past couple of decades with the development of \emph{targeted} doubly-robust methods yielding desirable theoretical efficiency guarantees on estimates for various causal effects \citep{van2011targeted,kennedy2016semiparametric}. These advancements have recently been incorporated into the neural network (NN) literature for causal inference via \emph{targeted regularization} (TR) \citep{shi2019adapting, nie2021vcnet}. TR methods produce favorable properties for causal estimation by incorporating a regularization term into a supervised neural network model. However, it remains an open task to develop a NN method that specifically targets the causal effect of a shift in the distribution for a continuous exposure/treatment variable \citep{munoz2012population}. We call this problem \emph{shift-response function} (SRF) estimation. Many scientific questions can be formulated as an SRF estimation task \citep{munoz2012population}. Some notable examples in the literature include estimating the health effects from shifts to the distribution of environmental, socioeconomic, and behavioral variables (e.g., air pollution, income, exercise habits) \citep{munoz2012population, diaz2020causal, smith2023application}.

One of our objectives is to develop a neural network technique that addresses a timely and highly prominent regulatory question. Specifically, the EPA is currently considering whether or not to revise the National Ambient Air Quality Standards (NAAQS), potentially lowering the acceptable annual average \PM concentration from 12 to 11, 10, or 9 \mug. We anticipate that the revision to the NAAQS will ultimately result in a shift to the distribution of \PM concentrations. Our goal is to estimate, for the first time, the reduction in deaths that would result from this anticipated shift using causal methods for SRFs.

\begin{figure}[tbp]
    \centering
    \vskip -4pt
    \includegraphics[width=0.75\linewidth]{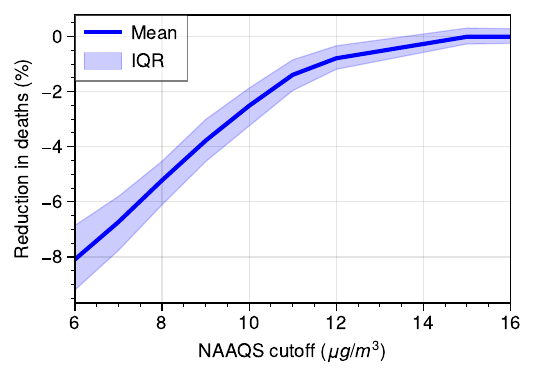}
    \vskip -10pt
    \caption{Estimated mortality reduction under a \emph{cutoff exposure shift} lowering the annual \PM in all regions below a given threshold. Uncertainty bands represent the interquartile range from an ensemble of networks. \emph{Data source}: US Medicare claims from 2000--2016.
    }
    \label{fig:results-cutoff}
    \vskip -6pt
\end{figure}

\emph{Contributions}.\quad We develop a novel method, called \textbf{T}argeted \textbf{R}egularization for \textbf{E}xposure \textbf{S}hifts with Neural \textbf{Net}works (\tresnet), which introduces two necessary and generalizable methodological innovations to the TR literature. First, we use a TR loss targeting SRFs, ensuring that our estimates retain the properties expected from TR methods such as asymptotic efficiency and double robustness \citep{kennedy2016semiparametric}. Given standard regularity conditions, these properties guarantee that the SRF is consistently estimated when either the outcome model or a density-ratio model for the exposure shift is correctly specified, achieving the best possible efficiency rate when both models are consistently estimated. Second, \tresnet accommodates non-continuous outcomes belonging to the exponential family of distributions (such as mortality counts) that frequently arise in real-world scenarios, including our motivating application. In addition to its suitability for our application, we evaluate the performance of \tresnet in a simulation study tailored for SRF estimation, demonstrating improvements over neural network methods not specifically designed for SRFs.

Our results contribute to the public debate informing the US Environmental Protection Agency (EPA) on the effects of modifying air quality standards. A preview of the results (fully developed in \cref{sec:application}) is presented in \cref{fig:results-cutoff}. The figure presents the estimated reduction in deaths (\%) resulting from various shifts to the distribution of \PM across every ZIP-code in the contiguous US between 2000 and 2016. These shifts limit the maximum concentration of \PM to the cutoff value for every ZIP-code that exceeds the cutoff, and otherwise leaves ZIP-codes that do not exceed the cutoff unchanged. We vary the cutoff in this SRF between 6 \mug and  16 \mug (x-axis). The y-axis represents the \% reduction in deaths that corresponds with each cutoff threshold. Notably, a NAAQS threshold of 9 \mug, would have had the effect of decreasing elder mortality by 4\%. These findings present a data-driven perspective on the potential health benefits of the EPA's proposal. 

\paragraph{Related work} 
Neural network-specific methods in causal inference can broadly be categorized into those focusing on estimating individualized effects (e.g., \cite{bica2020estimating, yoon2018ganite}) and those aimed at understanding marginal effects from a population. Various articles in the latter category -- to which our work belongs -- use semiparametric theory to derive theoretical guarantees tied to a specific target causal estimand. Such guarantees are generally focused on double robustness and asymptotic efficiency \cite{kennedy2016semiparametric, bang2005doubly,robins2000robust,bickel1993efficient}. The \emph{efficient influence function} (EIF) plays a pivotal role in this domain by characterizing the best possible asymptotic efficiency of estimators. The EIF is also recognized as Neyman orthogonal scores within the double machine learning literature \cite{kennedy2022semiparametric}.

{Targeted regularization} (TR) has emerged as a tool to incorporate EIF-based methods within deep learning. Notably, the \textsc{dragonnet} \cite{shi2019adapting} introduced TR in the context of binary treatment effects, demonstrating its potential for causal inference. Subsequently, the \textsc{vcnet} \citep{nie2021vcnet} extended the application of TR to the estimation of exposure-response functions (ERFs), showcasing its utility in estimating dose-response curves for continuous treatments and highlighting the adaptability of TR in addressing complex causal questions. Our work builds on the TR literature recognizing its unexplored potential for SRF estimation. Indeed, EIF methods have been applied to scenarios closely related to SRFs outside the neural-network literature under the stochastic intervention and modified treatment policy frameworks \cite{munoz2012population, diaz2021nonparametric}. However, their direct integration with targeted regularization has not been previously undertaken in the literature -- a gap which we address in this work.

Relating to the application, recent studies have investigated the causal effects of air pollution on mortality using ERFs \cite{wu2020evaluating,bahadori2022end,josey2023air}. Although these methods inform policy regarding the marginal effects of air pollution on mortality, they cannot adequately address questions suited for SRFs, which are crucial for analyzing the impact of air quality regulation changes. The distinction between ERFs and SRFs, and the significance of SRFs for our application and policy implications, are elaborated in \cref{sec:problem,sec:application}.

\section{Problem Statement: The Causal Effect of an Exposure Shift}\label{sec:problem}

We let $(A, Y, \mX)$ denote a unit from the target population, where $A \in \gA$ is a continuous exposure variable (also known as the treatment), $Y \in \gY$ is the outcome of interest, and $\mX \in \gX$ are covariates. We assume a sample of size $n$ and denote it as $\{\mX_i, A_i, Y_i\}_{i=1}^n$. For instance, in the application of interest, which is described in detail in \cref{sec:application}, $i$ represents a zip code location in a given year, $A_i$ represents its annual average concentration of \PM ($\mu g/m^3$), $Y_i$ is the number of deaths, and $\mX_i$ includes demographic and socioeconomic variables. For simplicity, we will omit the subscript $i$ unless required for clarity or when describing  statistical estimators from samples.

We assume the following general non-parametric structural causal model (SCM)\cite{pearl2009causality}:
\begin{equation}\label{eq:scm}
\begin{aligned}
    Y &= f_Y(\mX, A, U_Y),\\
    A &= f_A(\mX, U_A), \\
    \mX &= f_\mX(U_\mX),
\end{aligned}
\end{equation}
where $U_Y, U_A, U_\mX$ are exogenous variables. The functions $f_Y, f_A, f_\mX$ are deterministic and unknown. We will avoid measure-theoretic formulations to keep the presentation simple.

The quantities $Y^a=f_Y(\mX, a, U_Y)$ are known as potential outcomes. Potential outcomes can be divided into two parts. The factual outcome corresponds with the observed outcome when $a = A$. The counterfactual outcomes are hypothetical results at any point $a \in \gA$ where $a \neq A$. We will typically refer to $Y^a$ as the counterfactual, as SRFs are tooled for predicting the potential outcome from unobserved exposures (i.e. the counterfactuals), unless specifically stated otherwise.

There are two functions that are pivotal in defining and estimating causal effects. First, the \emph{outcome function} represents the expected (counter)factual outcome conditional on covariates, expressed as
$$
\mu(\vx, a) := \E[Y|\mX=x, do(A=a)] = \E[Y^a| \mX=\vx].
$$
Unless otherwise specified, all expectations are taken from the data distribution in \cref{eq:scm}. The second key function, the (generalized) \emph{propensity score}, is the distribution of the exposure conditional on the covariates, denoted as $p(a|\vx):=p(A=a|\mX=\vx)$. Note that $p$ is also used generically to represent any density function with the given arguments; the propensity score being a special case.

\begin{figure*}[tbp]
  \centering
  \begin{subfigure}[t]{0.3\textwidth}
    \includegraphics[width=1\linewidth]{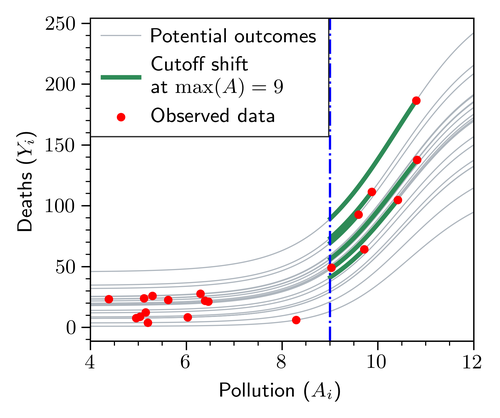}
    \vskip - 8pt
    \caption{Cutoff shift}
    \label{fig:toy:cutoff}
    \end{subfigure}%
  \begin{subfigure}[t]{0.3\textwidth}
    \includegraphics[width=1\linewidth]{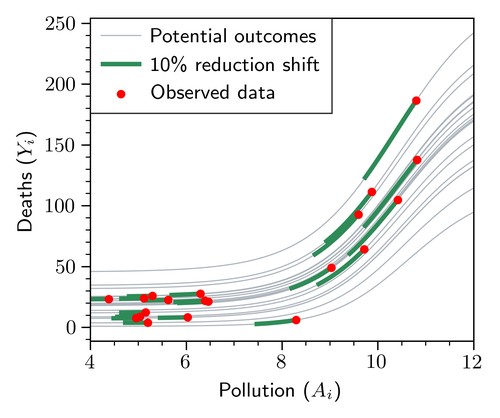}
    \vskip - 8pt
    \caption{Percent reduction shift}
    \label{fig:toy:percent}
    \end{subfigure}%
  \begin{subfigure}[t]{0.3\textwidth}
    \includegraphics[width=1\linewidth]{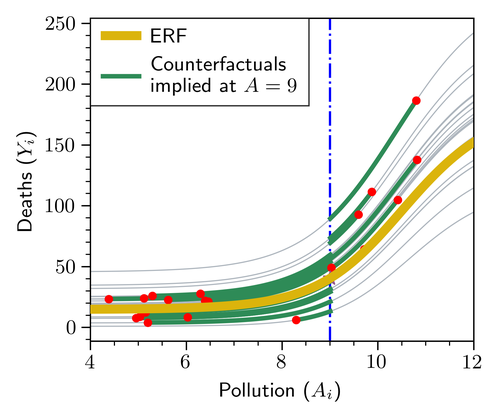}
    \vskip - 8pt
    \caption{ERF}
    \label{fig:toy:erf}
    \end{subfigure}%
  \caption{Two examples of exposure shifts with their implied counterfactuals and, for comparison, the implied counterfactuals of an exposure-response function at a given treatment value.}
  \label{fig:toy}
  \vskip -4pt
\end{figure*}

\paragraph{Exposure shifts} An exposure shift alters a unit's exposure from its observed value $A$ to a modified value $\tilde{A}$. To help build some intuition, one example of an exposure shift is a \emph{cutoff shift}, defined as $\tilde{A} = \min\{A, c\}$, which caps the exposure to a maximum threshold $c\in\R$. Another example is a percent reduction shift $\tilde{A}= cA$ in which, for instance, a value of $c=0.9$ would signify a $10\%$ reduction to the observed exposures. These shifts are depicted in \cref{fig:toy}. As in our application of interest, developed in detail in \cref{sec:application}, $\tilde{A}$ represents the (hypothetical) annual exposure to \PM that would have occurred in a particular zip code and year if an alternate policy (relative to the current NAAQS) had been implemented.

Formally, we define an exposure shift as a counterfactual exposure distribution $\tilde{A}\sim \tilde{p}(\tilde{A} | \mX)$, replacing the propensity score in the observed data distribution. We do not assume any advanced knowledge of the shifted distribution $\tilde{p}$ nor the analytical form of the shift. Instead, we can observe samples from the shifted distribution, realized as pairs of unshifted-shifted exposures $(A, \tilde{A})$. An alternative representation is that of a \emph{modified treatment policy} \cite{diaz2021nonparametric}, where the exposure shift can arise from a (possibly unknown) stochastic or deterministic transform -- $\tilde{A} = d(A, \mX)$. 

Importantly, while $A$ and $\tilde{A}$ may seem like different objects due to the notation, they in fact refer to the same data element -- the exposure or treatment. The tilde notation ``$\sim$" is helpful to emphasize whether it's being sampled from the shifted or unshifted distributions. 

\paragraph{The estimand of interest: the SRF}
The quantity of interest $\psi$ is the expected counterfactual outcome induced by the exposure shift after replacing $p(A|\mX)$ with $\tilde{p}(\tilde{A}|\mX)$:
\begin{equation}\label{eq:estimand}
\psi :=  {\E\left[Y^{\tilde{A}}\right]} = \E_{\mX \sim p(\mX)}\left[\E_{\tilde{A} \sim {\tilde{p}(\tilde{A}|\mX)}} \left[ {\mu(\mX, \tilde{A})} \mid \mX\right]\right].
\end{equation}
The SRF is a key estimand in the context of exposure shifts, as it allows us to estimate the effect of a policy-relevant shift in the distribution of a continuous exposure. 

Under some regularity conditions, established later in this section, the estimand can be rewritten using the importance sampling formula as 
\begin{equation}\label{eq:importance-sampling}
\begin{aligned}
\psi&=\E[\mu(\mX, A)w(\mX, A)]. \\
 w(\vx, a) &:= \tilde{p}(a|\vx) / p(a|\vx).
\end{aligned}
\end{equation}
We will derive estimators of $\psi$ using estimators of $\mu$ and $w$ in \cref{sec:tresnet}.

\paragraph{Comparison with traditional causal effects}
Exposure-response functions are the most common estimands in the causal inference literature for continuous exposures. Also known as a dose-response curve, an ERF can be described as the mapping $\xi(a) = \E_{\mX\sim p_{\mX}}[\mu(\mX, a)]$. 
 
One can consider ERFs as a special case of an SRF in which, for each treatment value $a\in\gA$, the quantity of interest is the average counterfactual outcome where the exposure is assigned to $a$ \emph{for all units}. That is, it can be seen as a limiting case of an SRF when $\tilde{p}$ is a degenerate point-mass distribution at $a$ for all units. This observation is the primary reason why SRFs are better suited for our motivating application as ERFs do not allow us to consider scenarios where each unit is given a different exposure value shifted from its observed value. A visual example is shown in \cref{fig:toy:erf}. At each point of the curve, an ERF describes the average outcome when \emph{all} units experience the \emph{same} \PM value. By contrast, SRFs allow us to estimate the average outcome when \emph{each} unit experiences a \emph{different} \PM value, as would occur under the NAAQS revision.
 
Moreover, ERFs require extrapolating the outcome estimates to regions where the positivity condition fails \cite{imbens_rubin_2015}, that is, regions of the joint exposure-covariate space where no data is available (see \cref{fig:toy:erf}). Extrapolation to such regions is unnecessary to estimate SRFs efficiently. This insight is supported by our simulation and experiment results in \cref{sec:experiments}.
  
\paragraph{Causal identification}
The target estimand $\psi$ can be expressed as a functional of the observable data distribution under standard assumptions, which are:

\begin{assumption}[Unconfoundedness]\label{as:unconf} There are no backdoor paths from $A$ to $Y$ in the SCM in \cref{eq:scm}, implying that $A \indep Y^a \mid \mX$ for all $a\in \gA$;
\end{assumption}
\begin{assumption}[Positivity]\label{as:pos} There are constants $ M_1, M_2 > 0$ such that $M_1 \leq p(a|\vx)/\tilde{p}(a|\vx) \leq M_2$ for all $(a, \vx)$ such that $\tilde{p}(a | \vx)>0$.
\end{assumption}

The first assumption ensures that $\rev{\E[Y^a | \mX=\vx]} = \E[Y | \mX=\vx, A=a]$. This result is sometimes known as causal identification since it allows us to express the causal effect of $A$ on $Y$ as a function of the observed data. A formal proof is given in the appendix. The second assumption implies that the density ratio in \cref{eq:importance-sampling} is well-defined and behaved. Notice that the notion of positivity used here, as required by SRFs, is generally weaker than the typical positivity assumption in the standard causal inference literature requiring $M_1 \leq p(a|\vx) \leq M_2$ for all $(\vx, a)$ \cite{munoz2012population}. 

\paragraph{Multiple shifts}
Under the framework outlined so far, it is possible to estimate the effect of multiple exposure shifts simultaneously. We simply let $\tilde{p} \in \tilde{\gP}$ denote the set of finite exposure shifts of interest and adopt the vectorized notation $\vw=(w_{\tilde{p}})_{\tilde{p}\in \tilde{\gP}}$, $\vpsi=(\psi_{\tilde{p}})_{\tilde{p}\in \tilde{\gP}}$, and $\tilde{\mA} = (\tilde{A}^{\tilde{p}})_{\tilde{p} \in \tilde{\gP}}$.

\section{Estimating SRFs with Targeted Regularization for Neural Nets}\label{sec:tresnet}

As we have suggested earlier, an estimator of the SRF can be derived from estimators of the outcome and density ratio functions. Using TR, we will obtain an estimator $\hat{\vpsi}\tr$ that ``converges efficiently'' to the true $\vpsi$ at a rate in accordance with the prevailing semiparametric efficiency theory surrounding robust causal effect estimation \citep{kennedy2022semiparametric}. 

\paragraph{Additional notation} To describe the necessary theoretical results based on semiparametric theory, we require some additional notation. Given a family of functions $f\in\gF$, denote $\lVert \gF\rVert_\infty=\sup_{f\in\gF}\lVert f\rVert_\infty$. The sample Rademacher complexity is denoted as $\Rad(\gF)=\sup_{f\in\gF}|\frac{1}{n}\sum_{i=1}^n \sigma_i f(U_i)|$ where $\sigma_i$ are independent and identically distributed Rademacher random variables satisfying $p(\sigma_i=1)=p(\sigma_i=-1)=1/2$. The Rademacher complexity measures the ``degrees of freedom" of an estimating class of functions. We use $O_p$ and $o_p$ to denote stochastic boundedness and convergence in probability, respectively.  Given a random variable $U \sim \sD$, denote $\sD (f) := \E_{U\sim \sD}[f(U)]$.  We also let $\sD_n$ denote the empirical distribution of $\sD$ given an iid sample $\{U_i\}_{i=1}^n$. It follows that $\sD_n(f)=\frac{1}{n}\sum_{i=1}^nf(U_i)$. Lastly, let $\mO=(\mX, A, \tilde{\mA}, Y) \sim \sP$ denote an element of the augmented data generating process.

\emph{The efficient influence function (EIF)}\quad The EIF is a fundamental function in semiparametric theory \citep{kennedy2022semiparametric}. It is denoted as the function $\vvarphi$ for an observation $\mO$. The EIF characterizes the gradient of the target estimand with respect to a small perturbation in the data distribution. The form of the EIF is specific to the target estimand. For the SRF in \cref{eq:estimand}, the EIF is given by
 \begin{equation}\label{eq:eif-def}
  \begin{aligned}
  \vvarphi(\mO; \vpsi, \mu, \vw) := \vw(\mX, A)\left(Y-\mu(\mX, A)\right) + \mu(\mX, \tilde{\mA}) - \vpsi.
  \end{aligned}
  \end{equation}
We provide a formal derivation in the appendix. The reader can refer to \citet{tsiatis2006semiparametric} and \citet{kennedy2022semiparametric} for a more comprehensive introduction to the nature and utility of influence functions. The importance of EIFs for causal estimation relies on the following two properties. First, the best possible variance among the family of regular, asymptotically linear estimators of $\vpsi$ is bounded below by $\sP(\vvarphi\vvarphi^\top)$. This lower bound is known as the efficiency rate. Second, the EIF can be constructed as a function of $(\vpsi, \mu, \vw)$. As it turns out, if a tuple of estimators $(\hat{\vpsi}, \hat{\mu}, \hat{\vw})$ satisfies the empirical estimating equation (EEE) given by $\sP_n(\vvarphi(\hat{\vpsi}, \hat{\mu}, \hat{\vw}))=0$, then $\hat{\vpsi}$ achieves the efficiency rate asymptotically.

The following observation offers an interpretation to the EEE. If $(\hat{\vpsi}, \hat{\mu}, \hat{\vw})$ satisfy the EEE, then $\hat{\vpsi}$ can be decomposed in terms of a debiasing component of the residual error and a plugin estimator for the marginalized average of the mean response:
\begin{equation}\label{eq:debiasing}
  \hat{\vpsi} = {\color{blue}\underbrace{\textstyle{\frac{1}{n}\sum_{i=1}^n\hat{\vw}(\mX_i, A_i)(Y_i-\hat{\mu}(\mX_i, A_i))}}_\text{debiasing term}} + {\color{red}\underbrace{\textstyle{\frac{1}{n}\sum_{i=1}^n \hat{\mu}(\mX_i, \tilde{\mA}_i)}}_\text{\color{red} plugin estimator}}
\end{equation}
This heuristic is crucial for the development of the TR loss in the next section. 

\subsection{Targeted Regularization for SRFs}\label{sec:tr-srf}

An immediate approach to obtain a doubly-robust estimator satisfying the EEE would be to use the right-hand side of \cref{eq:debiasing} as a definition of an estimator given nuisance function estimators $\hat{\mu}$ and $\hat{\vw}$. Such an estimator is called the \emph{augmented inverse-probability weighting} (AIPW) estimator for exposure shifts in analogy to the AIPW estimators for traditional average causal effects (ATE and ERF)~\citep{robins2000robust, robins2000profile}. However, it's been noted in the causal inference literature that, empirically, relying on the debiasing term of \cref{eq:debiasing} causes AIPW-type estimators to have suboptimal performance in finite samples \cite{shi2019adapting,van2011targeted}. Instead, targeted learning seeks an estimator that, by construction, satisfies the EEE without requiring the debiasing term. TR achieves this goal by learning a perturbed outcome model $\tilde{\mu}$ using a special regularization loss. 

\emph{Generalized TR for outcomes in the exponential family}.\quad 
We will derive the TR loss for SRFs while also extending the TR framework to accommodate non-continuous outcomes from the exponential family of distributions.

First, we say that the outcome follows a conditional distribution from the exponential family if $p(Y | \mX, A) \propto \exp(Y\eta(\mX, A) - \Lambda(\eta(\mX, A))$ for some function $\eta: \gX \times \gA \to \R$. The family's canonical link function $g$ is defined by the identity $g(\E[Y\mid \mX, A])=\eta(\mX, A)$.  For all distributions in the exponential family, $g$ is invertible \citep{mccullagh2019generalized}. Exponential families allow us to consider the usual mean-squared error and logistic regression as special cases. They also enable modeling count outcomes as in our motivating application. In this setting, we set $\Lambda(\eta)=e^\eta$ and $g(\mu)=\log(\mu)$ -- the canonical environment for Poisson regression. The following theorem forms the basis for the TR estimator. 

\begin{restatable}{theorem}{thmtwo}  
  \label{thm:two}
  Let $\vepsilon$ denote a perturbation parameter and define
  \begin{equation}\label{eq:glm-loss}
  \begin{aligned}
  & \gL\tr(\mu\nn, \vw\nn, \vepsilon)(\mO) = \\
  & \hspace{0.5in} \Lambda(g(\mu\nn(\mX, A)) + \vepsilon) - (g(\mu\nn(\mX, A)) + \vepsilon)Y. \\
  & \gR\tr(\mu\nn, \vw\nn, \vepsilon) =  \textstyle{\sP_n(\gL\tr(\mu\nn, \vw\nn, \vepsilon))}.
  \end{aligned}
  \end{equation}
   Then $(\frac{\partial\gR\tr}{\partial\vepsilon})({\mu}\nn, {\vw}\nn, {\vepsilon})=0$ if and only if
   $$\textstyle{\frac{1}{n}\sum_{i=1}^n{\vw}\nn(\mX_i, A_i)(Y_i-g^{-1}(g({\mu}\nn(\mX_i, A_i)) + {\vepsilon}))) = 0.}$$
  \end{restatable}

\begin{proof} A key property of an exponential family is $\Lambda'(\eta) =\frac{d}{d\eta}\Lambda(\eta)=g^{-1}(E[Y|\eta])$ \citep{mccullagh2019generalized}. Noting that $\frac{d}{d\vepsilon}g(\tilde{\mu}\nn(\mX, A))=\vw\nn(\mX, A)$ for all $\mu\nn,\vw\nn\vepsilon$, and using the chain rule, we have:

\begin{equation}
  \begin{aligned}
   \vzero &= \frac{d}{d\vepsilon} \gR\tr(\hat{\mu}, \hat{\vw}, \hat{\vepsilon}) \\
   & =\frac{1}{n}\sum_{i=1}^n\frac{\diff}{\diff\vepsilon}\Big\vert_{\vepsilon =\hat{\vepsilon}}\left\{  \Lambda(g(\tilde{\mu}(\mX, A))
   )  - Yg(\tilde{\mu}(\mX, A))\right\}\\
   &  =\frac{1}{n}\sum_{i=1}^n \left\{ g^{-1}(g(\tilde{\mu}(\mX, A)))\hat{\vw}(\mX, A) - Y \hat{\vw}(\mX, A)\right\} \\
   & = \frac{1}{n}\sum_{i=1}^n \hat{\vw}(\mX, A)(\tilde{\mu}(\mX, A) - Y).\\
  \end{aligned}
\end{equation}
The fact that $\hat{\vpsi}\tr=\frac{1}{n}\sum_{i=1}^n \tilde{\mu}(\mX_i, \tilde{A}_i)$ satisfies the empirical estimating equation follows trivially from the fact that $\sP_n\varphi(\hat{\vpsi}\tr, \tilde{\mu}, \hat{\vw})= \frac{1}{n}\sum_{i=1}^n \hat{\vw}(\mX, A)(Y - \tilde{\mu}(\mX, A)) + \frac{1}{n}\sum_{i=1}^n \tilde{\mu}(\mX_i, \tilde{A}_i) - \hat{\vpsi}\tr$. The first term is zero because of the above results while the last two terms cancel each other by definition.
\end{proof}

Notice that the condition ${\partial\gR\tr}/{\partial\vepsilon}=0$ in the theorem is achieved at the local minima of $\gR\tr$. Motivated by this observation, we next define the total TR loss as
$$
\gR(\mu\nn, \vw\nn, \vepsilon) := \gR_\mu(\mu\nn) + \alpha \gR_\vw(\vw\nn) + \beta_n\gR\tr(\mu\nn, \vw\nn, \vepsilon)
$$
where $\gR_\mu$ and $\gR_\vw$ are the empirical risk functions used to learn $\mu$ and $\vw$ (details in \cref{sec:architecture}), $\alpha>0$ is a hyperparameter, and $\beta_n$ is a regularization weight satisfying $\beta_n \to 0$. The latter condition is needed to ensure statistical consistency, as first discussed by \citet{nie2021vcnet} for ERF estimation. 

Since ${\vepsilon}$ only appears in the regularization term, the full loss in \cref{eq:total-tr-loss} preserves the result that $\frac{\partial\gR\tr}{\partial\vepsilon}=0$ upon optimization  Then, the TR estimator $\hat{\vpsi}\tr $ is defined as the solution of the optimization problem:
\begin{equation}\label{eq:total-tr-loss}
\begin{aligned}
(\hat{\mu}, \hat{\vw}, \hat{\vepsilon}) &=\argmin_{\mu\nn, \vw\nn, \vepsilon} \gR(\mu\nn, \vw\nn, \vepsilon)
\\
\hat{\vpsi}\tr & := \textstyle{\frac{1}{n}\sum_{i=1}^n g^{-1}(g(\hat{\mu}(\mX, A)) + \hat{\vepsilon}))}.
\end{aligned}
\end{equation}

In the theorem below we introduce the main result establishing the \emph{double robustness} and \emph{consistency} properties of our TR estimator:

\begin{restatable}{theorem}{thmthree}
\label{thm:three}
Let $\gM$ and $\gW$ be classes of functions such that ${\hat{\mu},\mu} \in \gM$ and ${\hat{\vw}, \vw} \in \gW$. Suppose assumptions \ref{as:unconf} and \ref{as:pos} hold, and that the following regularity conditions hold: (i) $\lVert \gM\rVert_\infty <\infty$, $\lVert \gW\rVert_\infty <\infty$, $\lVert 1/\gW\rVert_\infty <\infty$; (ii) \rev{either one model is correctly specified ($\hat{\mu}=\mu$ or $\hat{\vw}=\vw$), or both function classes have a vanishing complexity $\Rad(\gM)=O(n^{-1/2})$ and $\Rad(\gW)=O(n^{-1/2})$}; (iii) the loss function in \cref{eq:glm-loss} is Lipschitz; (iv) $\Lambda$ and $g$ are twice continuously differentiable. Then, the following statements are true:
\begin{enumerate}[itemsep=0pt,itemindent=0pt,topsep=0pt,leftmargin=18pt,parsep=0pt]
    \item The outcome and density ratio estimators of TR are consistent. That is, $\lVert \hat{\mu} - \mu \rVert_2 = o_p(1)$ and $\lVert \hat{\vw} - \vw \rVert_2 = o_p(1)$.
    \item The estimator $\hat{\vpsi}\tr$ satisfies
 $\lVert\hat{\vpsi}\tr - \vpsi \rVert_{\infty} = O_p(n^{-1/2} + r_1(n)r_2(n))$ whenever $\lVert \hat{\mu} - \mu\rVert_\infty = {O_p}(r_1(n))$ and $\lVert \hat{\vw}- \vw \rVert_\infty = {O_p}(r_2(n))$.
\end{enumerate}
\end{restatable}
\cref{thm:three} shows that the TR learner of the SRF achieves ``optimal'' root-$n$ convergence when $r_1(n)=r_2(n)=n^{-1/4}$ or when either $r_1$ or $r_2$ vanishes. Using standard arguments involving concentration inequalities, the Lipschitz assumption placed on the loss function can be relaxed by assuming that the loss function has a vanishing Rademacher complexity \citep{wainwright2019high}. 

\begin{proof}  The proof strategy follows \citet{nie2021vcnet} but is tailored to the SRF case and the more general exponential family of loss functions. We proceed in two steps.

\emph{Step 1: showing consistency of the outcome and density ratio models}.\quad  Denote the \emph{population risk} as
$$
\gR^*(\mu\nn, \vw\nn):=\sP\gL(\mu\nn, \vw\nn, \vzero),
$$
 where $\gL(\mu\nn, \vw\nn, \vzero)$ is the loss function in \cref{eq:glm-loss} \emph{without the targeted regularization}. The minimizers of $\gR^*$ are denoted $\mu^*$ and $\vw^*$. They are the estimators of $(\mu, \vw)$ under infinite data. We will now bound the risk loss, showing that the risk of the finite-sample estimators $(\hat{\vmu}, \hat{\vw})$ minimizing \cref{eq:glm-loss} is sufficiently close to that of $(\mu^*, {\vw}^*)$. Precisely, we shall show that
  \begin{equation}\label{eq:risk-consistency}
   \gR(\hat{\mu}, \hat{\vw}) - \gR(\mu^*, \vw^*)  = o(1) + O_p(n^{-1/2}). 
  \end{equation}
  To prove this fact, we first note that 
  \begin{equation}\label{eq:reg-risk}
    \begin{aligned} 
   0 & \leq \gR^*(\hat{\mu}, \hat{\vw}) - \gR^*(\mu^*, \vw^*)  \\
  & = (\gR(\hat{\mu}, \hat{\vw}, \vzero) - \gR(\mu^*, \vw^*, \vzero)) \\ 
    &\quad\hspace{1cm} (\sP - \sP_n)\gL(\hat{\mu}, \hat{\vw}, \vzero) + (\sP_n - \sP)\gL(\mu^*, \vw^*, \vzero).
  \end{aligned}
  \end{equation}
The first inequality is because of the definition of the risk minimizer and for the second one we added and subtracted the empirical risk and rearranged the terms.

Observe that the second and third terms in the last expression are empirical processes, suggesting we can use the uniform law of large numbers (ULLN). Indeed, the regularity assumptions on the Rademacher complexity provide a sufficient condition \citep{wainwright2019high}. Additionally, the order of the Rademacher complexity is preserved under Lipschitz transforms, then, using that the loss is Lipschitz, it follows that the class $\gF = \{\gL(\mu\nn, \vw\nn, \vzero)\colon \mu\nn \in \gM, \vw\nn \in \gW\}$ has a Rademacher complexity of order $O(n^{-1/2})$. Further, uniform boundedness is also preserved under Lispschitz transformations. Thus, the uniform law of large numbers applies, implying that the last two terms are $O_p(n^{-1/2})$. Note that the Lipschitz condition can be relaxed by the direct assumption that $\gL$ satisfies the boundedness and complexity conditions for the ULLN \citep{wainwright2019high}.

We now bound the first term. We begin by adding and subtracting the regularization terms, and compute:
  \begin{equation}\label{eq:reg-risk2}
    \begin{aligned} 
   & \gR(\hat{\mu}, \hat{\vw}, \vzero) - \gR(\mu^*, \vw^*, \vzero)  \\ 
   &\quad = \gR(\hat{\mu}, \hat{\vw}, \hat{\vepsilon}) - \gR(\mu^*, \vw^*, \vzero) \\
   &\quad\hspace{1cm} + \beta_n(\gR\tr(\mu^*, \vw^*, \vzero) - \gR\tr(\hat{\mu}, \hat{\vw}, \hat{\vepsilon})) \\
   &\quad \leq^{(a)} \beta_n(\gR\tr(\mu^*, \vw^*,\vzero) - \gR\tr(\hat{\mu}, \hat{\vw}, \hat{\vepsilon})) \\
   &\quad =^{(b)} \beta_n \sP_n(\gL_\mu(\mu^*) + O(1)).\\
   &\quad = \beta_n((\sP_n - \sP)\gL_\mu(\mu^*) + \gR_\mu(\mu^*) + O(1)) \\
   &\quad =^{(c)} \beta_n(O_p(n^{-1/2})) + O(1)) \\
   &\quad =^{(d)} O_p(n^{-1/2}) + o(1)  \\
  \end{aligned}
  \end{equation}
  Inequality (a) is due to $(\hat{\mu}, \hat{\vw}, \hat{\vepsilon})$ being a minimizer for the regularized risk; (b) is the result of $\gR\tr(\hat{\mu}, \hat{\vw}, \hat{\vepsilon})$ being bounded below since the exponential family of distributions is log-concave and $\gR\tr(\mu^*, \vw^*,\vzero) =\sP_n\gL_\mu(g(\mu^*))$; (c) uses the uniform law of large numbers from the Rademacher complexity and the uniform boundedness, and the fact that $\gL$ is Lipschitz; (d)
 follows from $\beta_n \to 0$. Combining \cref{eq:reg-risk} and \cref{eq:reg-risk2}, we get $\gR^*(\hat{\mu}, \hat{\vw}) - \gR^*(\mu^*, \vw^*) = o_p(1)$. 
  
  The result now follows from observing that the population risk has a unique minimizer up to the reparameterization of the network weights. Hence,  $\lVert \hat{\mu} - \mu \rVert_2 =o_p(1)$ and $\lVert \hat{\vw} - \vw \rVert_2 =o_p(1)$.

\emph{Step 2: Proving convergence and efficiency of $\hat{\vpsi}\tr$}.\quad Direct computation gives
\begin{equation}\label{eq:psi-diff}
\begin{aligned}
& \lVert \hat{\vpsi}\tr - \vpsi \rVert \\
& = \lVert \textstyle{\frac{1}{n}\sum_{i=1}^n \tilde{\mu}(\mX_i, \tilde{A}_i) } - \vpsi \rVert \\
& =^{(a)} \lVert \textstyle{\frac{1}{n}\sum_{i=1}^n \{ \tilde{\mu}(\mX_i, \tilde{A}_i)  + \hat{\vw}(\mX_i, A_i)(Y_i - \tilde{\mu}(\mX_i, A_i))\}} -  \vpsi \rVert \\
& =^{(b)} \lVert \E[\hat{\vw}(\mX, A)(Y - \tilde{\mu}(\mX, A))+ \tilde{\mu}(\mX, \tilde{A}) ] -  \vpsi  \rVert + O_p(n^{-1/2}) \\
& =^{(c)} \lVert \E[\hat{\vw}(\mX, A)(\mu(\mX, A) - \tilde{\mu}(\mX, A))+ \tilde{\mu}(\mX, \tilde{A})] -  \vpsi  \rVert + O_p(n^{-1/2}) \\
& =^{(d)} \lVert \E[(\hat{\vw}(\mX, A) - \vw(\mX, A))(\mu(\mX, A) - \tilde{\mu}(\mX, A))]   \rVert + O_p(n^{-1/2}), \\
\end{aligned}
\end{equation}
where (a) is by the property of the targeted regularization, namely, $\frac{1}{n}\sum_{i=1}^n \hat{\vw}(\mX_i, A_i)(Y_i - \tilde{\mu}(\mX_i, A_i))=0$; (b) is because of the uniform concentration of the empirical process, again using the vanishing Rademacher complexity and uniform boundedness; (c) integrates over $y$; (d) uses the definition of $\psi$ and the importance sampling formula with $\vw$. Since the link function $g$ is continuously differentiable, invertible and strictly monotone, then by the mean value theorem there exists $\vepsilon' \in (0, \hat{\vepsilon})$ such that
$$
\begin{aligned}
\tilde{\mu}(\mX, A) &=g^{-1}(g(\hat{\mu}(\mX, A)) +\hat{\vepsilon}) \\
& = \hat{\mu}(\mX, A) +  (g^{-1})'(g(\hat{\mu}(\mX, A)) + \vepsilon')\hat{\vepsilon}. \\
\end{aligned}
$$
From the uniform boundedness and smoothness of the link function, we have that $\hat{c} =(g^{-1})'(g(\hat{\mu}(\mX, A) + \vepsilon') < C$ for some constant $C > 0$. Then, using the above result in the last term of \cref{eq:psi-diff}, we obtain
\begin{equation}\label{eq:check-mate}
  \begin{aligned}
 &\lVert \E[(\hat{\vw}(\mX, A) - \vw(\mX, A))(\mu(\mX, A) - \tilde{\mu}(\mX, A))]   \rVert\\
  & \quad \leq \E[(\hat{\vw}(\mX, A) - \vw(\mX, A))(\mu(\mX, A) - \hat{\mu}(\mX, A))] \\
  &\quad\hspace{1cm} + C \lVert (\hat{\vw}(\mX, A) - \vw(\mX, A)) \hat{\vepsilon}\rVert \\
  & \quad \leq O_p(r_1(n)r_2(n)) +  O_p(r_2(n))\lVert  \hat{\vepsilon} \rVert 
  \end{aligned}
  \end{equation}
To complete the proof, we will show that $\lVert\hat{\vepsilon}\rVert  = O_p(r_1(n)) + O_p(n^{-1/2})$. Letting $\hat{c}_i$ be as in the Taylor expansion above, we can re-arrange the targeted regularization condition such that
$$
\begin{aligned}
\vzero &= \frac{d}{d\vepsilon} \gR\tr(\hat{\mu}, \hat{\vw}, \hat{\vepsilon}) \\
&= \frac{1}{n}\sum_{i=1}^n \hat{\vw}(\mX, A)(\hat{\mu}(\mX, A) - Y) + \frac{1}{n}\sum_{i=1}^n \hat{\vw}(\mX_i, A_i) \hat{c}_i\hat{\vepsilon}.
\end{aligned}
$$
Hence, we can write $\hat{\vepsilon}$ with the closed-form expression
$$
\begin{aligned}
\hat{\vepsilon} &= \argmin_{\vepsilon} \gR\tr(\hat{\mu}, \hat{\vw}, \vepsilon) = \frac{n^{-1}\sum_{i=1}^n \hat{\vw}(\mX_i, A_i)(Y_i -  \hat{\mu}(\mX_i, A_i))}{n^{-1}\sum_{i=1}^n\hat{c}_i\hat{\vw}(\mX_i, A_i)^2}.
\end{aligned}
$$
Observe now that $\hat c_i$ is uniformly lower bounded since the denominator is uniformly bounded in a neighborhood of the solution due to the assumption $\lVert 1/\gW \rVert_\infty < \infty$, and $g$ is strictly monotone and continuously differentiable. Hence, there is $C'>0$ such that 
\begin{equation}
\label{eq:epsilon-bound}
\begin{aligned}
\lVert \hat{\vepsilon} \rVert & \leq C' \lVert \textstyle{n^{-1}\sum_{i=1}^n \hat{\vw}(\mX_i, A_i)(Y_i -  \hat{\mu}(\mX_i, A_i))}\rVert\\
& \leq^{(a)} C' \lVert \E[\hat{\vw}(\mX, A) (Y -  \hat{\mu}(\mX, A)) \rVert + O_p(n^{-1/2}) \\
&\leq^{(b)}  O_p(r_1(n)) + O_p(n^{-1/2}),
\end{aligned}
\end{equation}
where (a) uses the uniform concentration of the empirical process and (b) again uses the uniform boundedness of $\hat{\vw}$. The proof now follows from combining equations (\ref{eq:psi-diff}), (\ref{eq:check-mate}), and (\ref{eq:epsilon-bound}).
\end{proof}

\begin{figure}[tbp]
  \centering
  \vskip -4pt
  \begin{subfigure}[t]{0.5\textwidth}
    \includegraphics[width=0.95\linewidth]{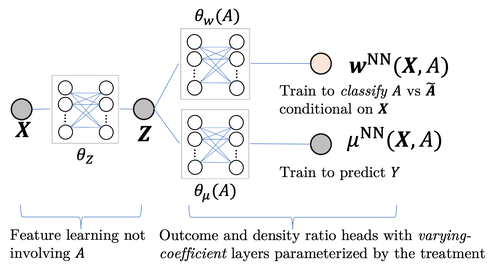}
    \end{subfigure}%
    \vskip - 8pt
  \caption{\tresnet architecture using a head for the density ratio model and a head for the outcome model.
  }
  \label{fig:architecture}
\end{figure}

\begin{table*}[tb]
    \centering
     \small %
   \caption{Experiment results. The table shows the $\sqrt{\textsc{mise}}$ across 100 random seeds with 95\%  confidence intervals computed with the asymptotic normal formula.}
     \begin{subtable}{\linewidth}\centering
    \begin{tabular}{l|rrrr|rrrr}
        \toprule
        & \multicolumn{4}{c|}{\sc Spline-based varying coefficients} & \multicolumn{4}{c}{\sc Piecewise linear varying coefficients} \\
        \cmidrule(lr){2-5} \cmidrule(lr){6-9}
        \textsc{Dataset} & \sc aipw$_\text{vc}$ & \sc outcome$_\text{vc}$ & \sc tresnet$_\text{vc}$$^*$ & \sc vcnet & \sc aipw$_\text{pl}$ & \sc drnet & \sc tresnet$_\text{pl}$$^*$ & \sc drnet+tr$_\text{erf}$ \\
        \midrule
        \sc ihdp & 3.15 (0.37) & 2.19 (0.06) &\bf  0.61 (0.03) & 0.63 (0.03) & 1.18 (0.14) & 2.36 (0.06) &\bf  0.15 (0.02) & 0.19 (0.02) \\
        \sc news & 1.5 (0.19) & 3.65 (0.04) & \bf 0.18 (0.02) & 0.28 (0.03) & 0.99 (0.12) & 0.99 (0.1) & \bf 0.17 (0.01) & 0.26 (0.03) \\
        \sc sim-B & 4.1 (0.57) & 0.5 (0.05) & \bf 0.26 (0.03) & 0.29 (0.04) & 1.46 (0.2) & 1.6 (0.2) & \bf 0.14 (0.02) & 0.16 (0.02) \\
        \sc sim-N & 5.69 (0.64) & 0.52 (0.05) & \bf 0.32 (0.02) & \bf 0.32 (0.03) & 1.81 (0.25) & 0.95 (0.06) & \bf 0.14 (0.01) & 0.15 (0.01) \\
        \sc tcga-1 & 1.13 (0.08) & 0.63 (0.02) &\bf  0.8 (0.01) & 0.87 (0.03) & 0.76 (0.05) & 0.62 (0.02) & \bf 0.61 (0.02) & 0.69 (0.03) \\
        \sc tcga-2 & 0.76 (0.09) & 0.24 (0.02) & \bf 0.18 (0.01) & 0.24 (0.02) & 0.36 (0.05) & 0.17 (0.01) & \bf 0.12 (0.0) & 0.16 (0.01) \\
        \sc tcga-3 & 0.83 (0.11) & 0.38 (0.03) & \bf 0.1 (0.01) & 0.15 (0.02) & 0.59 (0.06) & 0.59 (0.04) &\bf 0.08 (0.01) & 0.15 (0.02) \\
        \bottomrule
    \end{tabular}
    \caption{Performance of \tresnet in two varying-coefficient architectures.}
    \label{tbl:bench-arch}
    \end{subtable}\\
    \begin{subtable}{\linewidth}\centering
    \begin{tabular}{l|rr|rr}
        \toprule
        experiment & \sc  outcome$_\text{vc}$w/poisson loss$^*$ & \sc  outcome$_\text{vc}$w/mse loss & \sc tresnet$_\text{vc}$w/poisson loss$^*$ &\sc tresnet$_\text{vc}$w/mse loss \\
        \midrule
        \sc ihdp & \bf 18.82 (3.18) & 3726.92 (357.31) & \bf 2.04 (0.08) & 10986.43 (211.99) \\
        \sc news & \bf 3.41 (0.25) & 372.24 (52.02) & \bf 0.33 (0.05) & 1187.94 (100.8) \\
        \sc sim-B & \bf 1222.61 (1269.53) & 8433.98 (1327.22) &\bf  1113.58 (1270.49) & 16902.41 (1233.54) \\
        \sc sim-N & \bf 50.72 (4.45) & 2491.63 (310.36) & \bf 4.6 (0.22) & 18062.85 (243.2) \\
        \sc tcga-1 & \bf 184.89 (6.67) & 6682.91 (474.32) & \bf 40.56 (19.2) & 16307.15 (232.64) \\
        \sc tcga-2 & \bf 48.3 (1.44) & 5854.08 (483.41) & \bf 398.82 (143.96) & 16112.02 (186.99) \\
        \sc tcga-3 & \bf 18.27 (2.73) & 14381.06 (516.25) & \bf 12.92 (2.3) & 8565.17 (447.13) \\
        \bottomrule
    \end{tabular}
    \caption{Performance of \tresnet with and without the Poisson GLM formulation for count-based outcomes.}
    \end{subtable}
\end{table*}

\subsection{Architectures for Estimating the Outcome and Density Ratio Functions}\label{sec:architecture}

{Neural network architectures for nuisance function estimation have been widely investigated in causal inference (see \citet{farrell2021deep} for a review). We use the architectures proposed in the TR literature as a building block, particularly for continuous treatments \citep{nie2021vcnet}. Nevertheless, previous work in TR has not yet investigated the architectures required for SRF estimation. In particular, we require a new architecture to estimate the density ratio $\vw$. Rather, previous works have primarily focused on architectures for estimating propensity scores as required by traditional causal effect estimation (e.g. for estimating ATEs and ERFs).}

We describe a simple yet effective architecture for estimating $\mu$ and $\vw$. To keep the notation simple, we will write $f_\theta$ to denote generic output from a neural network indexed by weights $\theta$. The architecture has three components, illustrated in \cref{fig:architecture}. The first one maps the confounders $\mX$ to a latent representation $\mZ=f_{\theta_Z}(\mX) \in \R^d$. This component will typically be a multi-layer perception (MLP). The second and third components are the outcome and density-ratio heads, which are functions of $\mZ$ and the treatment. We describe all three components in detail below.
  
  \textbf{Outcome model}\quad Recall that we assume the outcome $Y$ follows a conditional distribution from the exponential family. That is, $p(Y|\mX, A)\propto \exp(Y\eta(\mX, A) - \Lambda(\eta(\mX, A))$ with an invertible link function  $g$ satisfying $\mu(\mX, A)=g^{-1}(\eta(\mX, A))$. We can identify the canonical parameter $\eta$ with the output of the neural network and learn $\hat{\mu}$ by minimizing the empirical risk
  \begin{equation}
    \label{eq:muloss}
    \begin{aligned}
\gR_\mu(\mu\nn)=\frac{1}{n} \sum_{i=1}^n \left\{\Lambda(g^{-1}(\mu\nn(\mZ_i, A_i)) - Yg^{-1}(\mu\nn(\mZ_i, A_i))\right\}.
    \end{aligned}
    \end{equation}
    
Next, we need to select a functional form for the neural network. An MLP parameterization with the concatenated inputs of $(\mZ, A)$--the na\"ive choice--would likely result in the effect of $A$ being lost in intermediate computations. Instead, we adopt the \emph{varying coefficient} approach by setting $\mu\nn(\mX, A)=g^{-1}(f_{\theta_\mu(A)}(\mZ))$ \citep{nie2021vcnet,chiang2001smoothing}. With this choice, the weights of each layer are dynamically computed as a function of $A$ obtained from a linear combination of \underline{basis} functions spanning the set of admissible functions on $A$. The weights of the linear combination are themselves a learnable linear combination of the hidden outputs from the previous layer.
We refer the reader to \citet{nie2021vcnet} for additional background on varying-coefficient layers. Our experiments suggest TR is beneficial for different choices of basis functions.

\emph{Estimation of $\vw$ via classification}.\quad The density ratio head $\vw\nn=(w\nn_{\tilde{p}})_{\tilde{p}\in \tilde{\gP}}$ is trained using an auxiliary classification task. For this purpose, we use an auxiliary classification task where the positive labels are assigned to the samples from $\tilde{A}^{\tilde{p}}$ and the negative labels to the samples with $A$. This loss can be written as:

\begin{equation}
  \label{eq:wloss}
  \begin{aligned}
 \gR_w(\vw\nn) &= \frac{1}{2n|\tilde{\gP}|} \sum_{i=1}^n\sum_{\tilde{p} \in \tilde{\gP}}\Big\{
 B(\zeta_{1,i}^{\tilde{p}}, \vone)+ B(\zeta_{0,i}, \vzero) 
   \Big\} \\
    \zeta_{1,i}^{\tilde{p}} &= \log w\nn_{\tilde{p}}(\mX_i, \tilde{A}_i^{\tilde{p}}) \\
   \zeta_{0,i} &= \log w\nn_{\tilde{p}}(\mX_i, A_i)
  \end{aligned}
\end{equation}
where $B$ is the binary classification loss and $\zeta_{0,i}, \{\zeta_{1,i}^{\tilde{p}}\}_{\tilde{p}\in\tilde{\gP}}$ are the label predictions on the logit scale, which coincide with the log-density ratios \cite{sugiyama2012density}. 
To capture the effect of $A$ more effectively, we propose parameterizing the network using the varying-coefficient structure discussed in the previous section with $\log w_{\tilde{p}}\nn(\mX, A)=f_{\theta^{\tilde{p}}_w(A)}(\mZ)$. To our knowledge, we are the first to consider a varying-coefficient architecture for conditional density ratio estimation.

\section{Simulation study}\label{sec:experiments}

We conducted simulation experiments to validate the design choices of \tresnet. The so-called fundamental problem of causal inference is that the counterfactual responses are never observed in real data. Thus, we must rely on widely used semi-synthetic datasets to evaluate the validity of our proposed estimators. First, we consider two datasets introduced by \citet{nie2021vcnet}, which are continuous-treatment adaptions to the popular datasets \textsc{ihdp} \citep{hill2011bayesian} and \textsc{news} \citep{newman2008bag}. We also apply our methods to \cite{nie2021vcnet}'s fully simulated data, \textsc{sim-N}. In addition, we consider the fully simulated dataset described in~\cite {bahadori2022end}, which features a continuous treatment and has been previously used for calibrating models in air pollution studies. Finally, we also consider three variants of the \textsc{tcga} dataset presented in \citet{bica2020estimating}. The three variants consist of three different dose specifications as treatment assignments and the corresponding dose-response as the outcome. 

The datasets described here have been employed without substantial modifications from the original source studies to facilitate fair comparison. Note that for each of these synthetic and semi-synthetic datasets, we have access to the true counterfactual outcomes, which allows us to compute sample SRFs exactly. We will consider the estimation task of 20 equally-spaced percent reduction shifts between 0-50\% from the current observed exposures. More specifically, $\tilde{A}=(1 - c)A$ for values of $c$ in $0-50\%$.

\emph{Evaluation metric and task}.\quad
 Given a semi-synthetic dataset $\gD$, consider an algorithm that produces an estimator $\hat{\vpsi}^{(s)}_{\tilde{p}, \gD}$ of ${\vpsi}^{(s)}_{\tilde{p}, \gD}$ given an exposure shift $\tilde{p}$ and random seed $s$. To evaluate the quality of the estimator, we use the \emph{mean integrated squared error}
$\textsc{mise}_\gD=(n_\text{seeds}|\tilde{\gP}|)^{-1} \sum_{s}\sum_{\tilde{p}}\lvert\hat{\vpsi}^{(s)}_{\tilde{p}, \gD} - \vpsi^{(s)}_{\tilde{p}, \gD} \rvert^2.$
This metric is a natural adaption of an analogous metric commonly used in dose-response curve estimation \citep{bica2020estimating}. 

\emph{Experiment 1: Does targeting SRFs specifically improve their estimation (compared to no targeting or standard ERF targeting?)}\quad 
We evaluate two variants of \tresnet against alternative SRF estimators, including \textsc{vcnet} \citep{nie2021vcnet} and \textsc{drnet} \citep{schwab2020learning}, two prominent methods used in causal TR estimation for continuous treatments. Since these methods were not designed specifically for SRFs, we expect that their performance for this task can be improved by adapting TR through their baseline architectures.

The first variant of \tresnet uses varying-coefficient layers based on splines---see the discussion in \cref{sec:architecture} for background. We compare this variant, named \tresnetvc, with the following baselines:
\begin{enumerate}
\item  $\textsc{aipw}_\textsc{vc}$, the AIPW-type estimator for SRFs discussed in \cref{sec:tr-srf}, wherein we fit separate outcome and density ratio models which are then substituted into \cref{eq:debiasing};
\item $\textsc{outcome}_\textsc{vc}$, which uses the same outcome model as \tresnetvc, but without the TR and density ratio heads;
\item \textsc{vcnet}, which uses a similar overall architecture as \tresnetvc, but with a TR designed for ERFs rather than for SRFs.
\end{enumerate}

The second variant of \tresnet uses varying coefficients based on piecewise linear functions instead of splines. This variant, \tresnetpl, is compared against the following analogous baselines: 
\begin{enumerate}
\item $\textsc{aipw}_\textsc{pl}$, which is the piecewise analogue to $\textsc{aipw}_\textsc{vc}$;
\item \textsc{drnet}, based on \citet{schwab2020learning}, functioning as the piecewise linear analogue to $\textsc{outcome}_\textsc{vc}$;
\item $\textsc{drnet}+\textsc{tr}_\textsc{erf}$, which is the analogue to \textsc{vcnet} that can be constructed by adding a density ratio head and TR designed for ERFs rather than SRFs.
\end{enumerate}

\cref{tbl:bench-arch} shows the results of this experiment. For both architectures, the \tresnet variants achieve the best performance. \tresnetvc is somewhat better than \textsc{vcnet}, which have comparable architectures although \tresnet uses a different TR implementation specific for SRF estimation rather than ERF estimation. Likewise, \tresnetpl outperforms $\textsc{drnet}+\textsc{tr}_\textsc{erf}$. These moderate but consistent performance improvements suggest the importance of SRF-specific forms of TR. We also see strong advantages against outcome-based predictions and AIPW estimators, suggesting that the TR loss and the shared learning architecture is a boon to performance. These results are compatible with observations from previous work in the TR literature \citep{nie2021vcnet, shi2019adapting}.

\emph{Experiment 2: Does TR improve estimation when count-valued outcomes are observed?}\quad For these experiments, we used the spline-based variant and evaluate whether \tresnet with the Poisson-specific TR, explained in \cref{sec:tresnet}, performs better than the mean-squared error (MSE) loss variant when the true data follows a Poisson distribution.  This evaluation is important since our application consists of count data requiring a Poisson model which are widely used to investigate the effects of \PM on health \citep{wu2020evaluating,josey2022estimating}. We construct similar semi-synthetic datasets as in Experiment 1, however in this experiment the outcome-generating mechanism samples from a Poisson distribution rather than a Gaussian distribution. The results of this experiment are clear--using the correct exponential family for the outcome model is critical, regardless of whether TR is implemented. 

\section{Main Application: The Effects of Stricter Air Quality Standards}\label{sec:application}

We implemented \tresnet for count data to estimate the health benefits caused by shifts to the distribution of \PM that would result from lowering the NAAQS---the regulatory threshold for the annual-average concentration of \PM enforced by the EPA.

\begin{figure}[tbp] %
    \centering
    \includegraphics[width=0.77\linewidth]{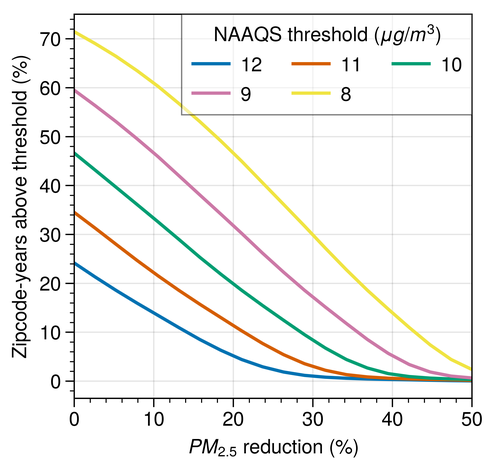}
    \vskip -8pt
    \caption{Fraction (\%) of observed units remaining above \PM limit as a function of reduction (\%) considering different NAAQS (current NAAQS is set at 12~\mug).}
    \label{fig:results-quantiles}
\end{figure}

\emph{Data}.\quad
The dataset is comprised of Medicare recipients\footnote{Access to Medicare data is restricted without authorization by the Centers for Medicare \& Medicaid Services since it contains sensitive health information. The authors have completed the required IRB and ethical training to handle these datasets.} from 2000--2016, involving 68 million unique individuals. The data includes measurements on participant race/ethnicity, sex, age, Medicaid eligibility, and date of death, which are subsequently aggregated to the annual ZIP-code level. The \PM exposure measurements are extracted from a previously fit ensemble prediction model \citep{di2019ensemble}. The confounders include measurements on meteorological information, demographics, and the socioeconomic status of each ZIP-code. Calendar year and census region indicators are also included to account for spatiotemporal trends. To compile our dataset, we replicated the steps and variables outlined by \citet{wu2020evaluating}.

\emph{Exposure shifts}.\quad
We consider two types of \PM shifts, \emph{cutoff shits} and \emph{percent reduction shift}, each providing different perspectives and insights. The counterfactuals implied from these scenarios are illustrated in figures \ref{fig:toy:cutoff} and \ref{fig:toy:percent}, respectively. First, a \emph{cutoff shift}, parameterized by a threshold $d$, encapsulates scenarios in which every ZIP-code year that exceeded some threshold are truncated to that maximum threshold. Mathematically, the shift is defined by the transformation $\tilde{A} = \min(A, c)$. To be more succinct, the exposure shift defines a counterfactual scenario. For this application, the threshold $c$ is evaluated at equally spaced points starting with 15 \mug moving down to 6 \mug. We expect that at 15 \mug there will be little to no reduction in deaths since $>99\%$ of observations fall below that range. We can contemplate the proposed NAAQS levels through $d$ assuming that full compliance to the new regulation holds for incompliant ZIP-codes. The exposure shift should otherwise not affect already compliant ZIP-codes. Second, we  onsider \emph{percent reduction shifts}. This scenario assumes that all ZIP-code years reduce their pollution levels proportionally from their observed value. More precisely, the shift is defined as $\tilde{A}=A(1 - c)$. We considered a range of percent reduction shifts between $c\in (0, 50)\%$. We can interpret these shifts in terms of the NAAQS by mapping each percent reduction to a compliance percentile. For instance, \cref{fig:results-quantiles} shows that, under a 30\% overall reduction in historical values, approximately 82\% would comply with a NAAQS of 9 \mug.

\emph{Implementation}.\quad
We implement \tresnet using varying-coefficient splines as in \cref{sec:experiments}. We select a NN architecture using the out-of-sample prediction error from a 20/80\% test-train split to choose the number of hidden layers (1-3 layers) and hidden dimensions (16, 64, 256). We found no evidence of overfitting in the selected models. To account for uncertainty in our estimations, we train the model on 100 bootstrap samples, each with random initializations, thereby obtaining an approximate posterior distribution. Deep learning ensembles have been previously shown to approximate estimation uncertainty well in deep learning tasks \citep{izmailov2021bayesian}.

\begin{figure}  %
  \centering
  \includegraphics[width=0.9\linewidth]{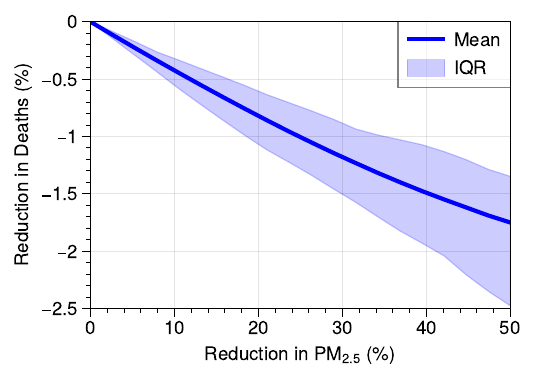}
  \vskip -8pt
  \caption{Estimated SRF of the total deaths (\%) for different cutoffs.}
  \label{fig:results-percent}
\end{figure}

\emph{Results}.\quad
\cref{fig:results-cutoff} in the introduction presents the effects of shifting the \PM distribution at various cutoffs on the expected reduction in deaths. The slope is steeper at stricter/lower cutoffs, likely because lower cutoffs affect a larger fraction of the observed population and reduce the overall \PM. For instance, figure \cref{fig:results-cutoff} shows that had no ZIP-code years exceeded 12 \mug, the observed death counts would have decreased by around 1\%. If the cutoff is lowered to 9 \mug, then deaths could have fallen by around 4\%. The slope becomes increasingly steeper as the \PM threshold is reduced, suggesting an increasing benefit from lowering the standard concentration level. Another way to interpret this result is to say that there is a greater gain to reducing mortality caused by \PM from lowering the concentration level from 10 to 8 \mug than there is from lowering it from 12 to 10 \mug, despite the obvious observation that both contrasts examine a \PM reduction of 2 \mug.

The results of the percent-reduction shift are presented in~\cref{fig:results-percent}. The decrease in deaths is approximately linear with respect to the percent decrease in \PM. As such, the SRF shows an approximate 0.5\% decrease in deaths resulting from a 10\% decrease in \PM. This result is consistent with previous causal estimates of the marginal effect of \PM exposure on elder mortality \citep{wu2020evaluating}. Percent reduction offers a complementary view to the cutoff shift response function as it might pertain to policymakers decisions on the future of the NAAQS. 

\section{Discussion and Limitations}\label{sec:discussion}
The results of \cref{sec:application} demonstrate the significant need to establish proper estimators when addressing the pressing public health question regarding the potential health benefits of lowering the NAAQS in the United States. In response to this question, we introduce the first causal inference method to utilize neural networks for estimating SRFs. Furthermore, we have extended this method to handle count data, which is crucial for our application in addition to other public health and epidemiology contexts. 

There are numerous opportunities to improve to our methodology. First, our uncertainty assessment of the SRF relies on the bootstrap and ensembling of multiple random seeds. While these methods are used often in practice, future research could explore the integration of \tresnet with Bayesian methods to enhance uncertainty quantification. Second, our application of the methodology focuses on exposure shifts representing complementary viewpoints to the possible effects of the proposed EPA rules on the NAAQS. However, it does not determine the most probable exposure shift resulting from the new rule's implementation, based on historical responses to changes in the NAAQS. Subsequent investigations should more carefully consider this aspect of the analysis. The assessment of annual average \PM concentrations at the ZIP-code level is based on predictions rather than on actual observable values, introducing potential attenuation bias stemming from measurement error. Nonetheless, previous studies on measurement error involving clustered air pollution exposures have demonstrated that such attenuation tends to pull the causal effect towards a null result \citep{josey2022estimating, wei2022impact}. 

It is essential to recognize that the SRF framework places additional considerations on the analyst designing the exposure shift. This newfound responsibility can be seen as both a disadvantage and an advantage. However, it highlights the need for an explicit and meticulous statement of the assumptions underlying the considered exposure shifts in order to mitigate the potential misuse of SRF estimation techniques.

\begin{acks}
The work was supported by the National Institutes of Health
(R01-AG066793, R01-ES030616, R01-ES034373), the Alfred P. Sloan Foundation (G-2020-13946), and the Ren Che Foundation supporting the climate-smart public health project in the Golden Lab. Computations ran on the FASRC Cannon cluster supported by the FAS Division of Science Research Computing Group at Harvard University.

\end{acks}

\bibliographystyle{ACM-Reference-Format}
\bibliography{sample-base}

\appendix

\section{Additional proofs}\label{sec:eif}

\paragraph{Causal identification}
We first show the identification result from \cref{sec:problem}, which establishes that the outcome function $\mu$ is equal to the conditional expectation of the outcome given the treatment and covariates. As a corollary, the SRF estimand $\vpsi$ can be estimated from observed data.

\begin{proposition}\label{prop:identification}
  Suppose \cref{as:unconf} holds. Then $\mu(\mX, A)=\E[Y \mid \mX, A]$. 
\end{proposition}
  \begin{proof} The proof is fairly standard in the causal inference literature \cite{imbens_rubin_2015}.
    Since the treatment and potential outcomes are independent conditional on $\mX$ (unconfoundedness), properties of the conditional expectation yield
    $$
    \begin{aligned}
    \mu(\vx, a)&=\E[Y^a \mid \mX=\vx]\\
    &= ^{(a)} \E[Y^a \mid \mX=\vx, A=a] \\
    &=\E[Y^A \mid \mX=\vx, A=a] =^{(b)} \E[Y \mid \mX=\vx, A=a].
    \end{aligned}
    $$
The (a) uses unconfoundedness (\cref{as:unconf}) and (b) $Y^A=Y$ by definition.
\end{proof}

\begin{corollary} The SRF $\vpsi$ is identified by the data distribution.
\end{corollary}
\begin{proof}
From the definitions and the importance sampling formula, we have
$$
\begin{aligned}
  \vpsi & := \E_{\mX \sim p(\mX)}[\E_{\tilde{\mA} \sim \tilde{p}(\tilde{\mA} | \mX)}[\mu(\mX, \tilde{\mA})| \mX]] \\
  &= \E_{\mX \sim p(\mX)}[\E_{\tilde{\mA} \sim \tilde{p}(\tilde{\mA} | \mX)}[E[Y | \mX, A=\tilde{\mA}]| \mX, \tilde{\mA}]] \\
\end{aligned}
$$
which shows identification as the right-hand side only involves an expectation over the observed data distribution by \cref{prop:identification}.
\end{proof}

\paragraph{Derivation of the EIF} 
When we fit a statistical model to a dataset, each data point contributes to the estimated parameters of the model. The efficient influence function (EIF) effectively measures how sensitive an estimate is to the inclusion or exclusion of individual data points. In other words, it tells us how much a single data point can influence the estimate of a causal effect. Formally, the EIF is defined as the canonical gradient \citep{van2011targeted} of the targeted parameter. Before defining it properly, we need to introduce a few objects and notations.  First, we let $\{\sP_t: t\in \R \}$ be a smooth parametric submodel, that is, a parameterized family of distributions such that $\sP_0=\sP$. Next, we adopt the following notation convention
$$
\begin{aligned}
  \tilde{p}_t(\vo) &:= p_t(y|\vx,\tilde{\va})\tilde{p}_t(\tilde{\va}|\vx)p_t(\vx) \\
  p_t(\vo)& := p_t(y|\vx,a){p}_t(a|\vx)p_t(\vx) \\
  w_t(\vx, \va) & := \tilde{p}_t(\va|\vx)/p_t(a|\vx)
\end{aligned}
$$
where the subscript $t$ indicates that the density is from the data distribution is taken from $\sP_t$. The score function for each member of the submodel is defined as
$$
s_t(\vo):=\frac{\diff}{\diff t}\Big\vert_{h=0}\log  \tilde{p}_{t + h}(\vo),
$$ 
Similarly, we can define the SRF estimand at each member of the submodel as
  $$
  \begin{aligned}
  \vpsi(\sP_t)=\int y \tilde p_t(\vo)\diff{\vo}. 
  \end{aligned}
  $$
It follows that $\vpsi(\sP_0)=\vpsi$ is the target SRF.  

Under a slight abuse of notation to reduce clutter and avoid multiple integrals, we denote here and in the proofs the integral of the relevant variables (understood from context) using a single operator $\int \diff{\vo}$. 

The EIF of the non-parametric model is the canonical gradient of the SRF estimand, which is characterized by the equation
\begin{equation}\label{eq:change}
\frac{\diff}{\diff t}\Big\vert_{t=0} \vpsi(\sP_t) = \int \vvarphi(\vo; \psi, \mu, \vw)s_0(\vo) \tilde{p}(\vo)\diff{\vo}
\end{equation}
This condition is also known as \emph{pathwise differentiability}. 

The following results establish that the EIF of the SRF is given by $\vvarphi(\vo; \vpsi, \mu, \vw)$.

\begin{restatable}{proposition}{thmone}
  \label{thm:one} Suppose Assumptions \ref{as:unconf} and \ref{as:pos} hold. Then the EIF of $\vpsi$ is given by
\begin{equation}
\begin{aligned}
\vvarphi(\mO; \vpsi, \mu, \vw)= \vw(\mX, A)\left(Y-\mu(\mX, A)\right) + \mu(\mX, \tilde{\mA}) - \vpsi.
\end{aligned}
\end{equation}

\end{restatable}

\begin{proof}
We need to show that \eqref{eq:change} holds. Starting on the left hand side, and using properties of logarithms, we have
\begin{equation}\label{eq:useful-score}
\begin{aligned}
  \int y s_t(\vo) \tilde{p}_t(\vo)\diff{\vo} &= \int y \frac{\diff{ \tilde{p}_t(\vo)}/\diff{t}}{\tilde{p}_t(\vo)} p_t(\vo)\diff{\vo} \\
  &= \frac{\diff}{\diff{t}} \int y \tilde{p}_t(\vo) \diff{\vo}.
\end{aligned}
\end{equation}

Observe now that (also due to logarithmic properties) we can factorize $s_{t}(\vo)$ into three parts:
\begin{equation}
\label{eq:score-factorization}
s_{t}(\vo) = s_{1,t}(y|\tilde{\va}, \vx) + s_{2,t}(\tilde{\va}|\vx) + s_{3,t}(\vx),
\end{equation}
where $s_{1,t}(y|\tilde{\va}, \vx)$, $s_{2,t}(\tilde{\va}|\vx)$, and $s_{3,t}(\vx)$ are the score functions for the outcome, shifted exposure, and covariates, respectively. Combining this decomposition with \cref{eq:useful-score}, we can rewrite the left-hand side of (\ref{eq:change}) as
\begin{equation} \label{eq:pathwise}
\begin{aligned}
     \frac{\diff}{\diff{t}} \psi(\sP_t) &= \int y s_{t}(\mathbf{o})\tilde{p}_t(\vo)\diff{\vo} \\
     &= \int y (s_{1,t}(y|\tilde{\va}, \vx) + s_{2,t}(\tilde{\va},|\vx) + s_{3,t}(\vx)) \tilde{p}_t(\vo)\diff{\vo} \\
     & \quad\quad + \int y s_{2,t}(\tilde{\va},|\vx)\tilde{p}_t(\vo)\diff{\vo} + \int y s_{3,t}(\vx) \tilde{p}_t(\vo)\diff{\vo}.
\end{aligned}
\end{equation}

For the next step of the proof, we will recursively use the two following identities. For any arbitrary function $g(\cdot)$, and for any two subsets of measurements $\vo_1$ and $\vo_2$ (e.g. $\vo_1 = (a, \vx)$ and $\vo_2 = y$), we have
\begin{equation}\label{eq:identities}
\begin{aligned}
    \int g(\vo_1)\left\{\vo_2 - \int \vo_2 \tilde{p}_t(\vo_2|\vo_1)\diff{\vo_2}\right\} \tilde{p}_t(\vo)\diff{\vo} &=^{(a)} 0 \\
    \int g(\vo_1)s_t(\vo_2|\vo_1)\tilde{p}_t(\vo)\diff{\vo} &=^{(b)} 0
\end{aligned}
\end{equation}
Using these identities and evaluating at $t=0$, we can manipulate the first term in (\ref{eq:pathwise}) as
$$
\begin{aligned}
     & \int y s_{1,t}(y|\tilde{\va}, \vx)\diff\tilde{p}_t(\vo)\diff\vo\Big\vert_{t=0} \\
     &\quad =^{(b)} \int \left\{y - \mu(\vx, \tilde{\va})\right\} s_{1,0}(y|\tilde{\va}, \vx)p(y|\tilde{\va}, \vx)\tilde{p}(\tilde{\va}|\vx)p(\vx)\diff{\vo} \\
     &\quad = \int \vw(\vx, a)\left\{y - \mu(\vx, a)\right\} s_{1,0}(y|a, \vx)p(y|a, \vx)p(a|\vx)p(\vx)\diff{\vo} \\
    &\quad =^{(a)} \int \vw(\vx, a) \\
    &\quad\hspace{1cm} \times \left\{y - \mu(\vx, a)\right\}\left\{s_{1,0}(y|a, \vx) + s_{2,0}(a|\vx) + s_{3,0}(\vx)\right\}p(\vo)\diff{\vo}\\
    &\quad =\int \vw(\vx, a)\left\{y - \mu(\vx, a)\right\} s_{0}(\vo)p(\vo)\diff{\vo}.
\end{aligned}
$$
where we used the importance sampling formula for the second equality. Now, for the second term in (\ref{eq:pathwise}), also evaluated at $t=0$, we have
$$
\begin{aligned}
    & \int y s_{2,t}(\tilde{\va}|\vx)\diff\tilde{p}_t(\vo)\diff\vo\Big\vert_{t=0} \\
    &\quad = \int  \mu(\vx, \tilde{\va}) s_{2,0}(\tilde{\va}|\vx) \tilde{p}(\tilde{\va}|\vx)p(\vx)\diff{\vo} \\
    &\quad =^{(b)} \int \left\{\mu(\vx, \tilde{\va}) - \int \mu(\vx, \tilde{\va}) \tilde{p}(\tilde{\va}|\vx)\diff{\vo} \right\} \\
    &\quad\hspace{2cm}\times \left\{ s_{1,0}(y|\tilde{\va}, \vx) + s_{2,0}(\tilde{\va}|\vx) \right\} \tilde{p}(\vo)\diff{\vo}  \\
    &\quad =^{(a)} \int \left\{\mu(\vx, \tilde{\va}) - \int \mu(\vx,\tilde{\va}) \tilde{p}(\tilde{\va}|\vx)\diff{\tilde{\va}} \right\}\\
    &\quad\hspace{2cm} \times \left\{ s_{1,0}(y|\tilde{\va}, \vx) + s_{2,0}(\tilde{\va}|\vx) + s_{3,0}(\vx) \right\} \tilde{p}(\vo)\diff{\vo}  \\
    &\quad = \int \left\{\mu(\vx, \tilde{\va}) - \int \mu(\vx, \tilde{\va}) \tilde{p}(\tilde{\va}|\vx)\diff{\tilde{\va}} \right\} s_0(\vo)\tilde{p}(\vo)\diff{\vo}  
\end{aligned}
$$
where we again center the integrand and recursively apply the identities of (\ref{eq:identities}). Finally, for the third term in (\ref{eq:pathwise}) we have
$$
\begin{aligned}
    &\int y s_{3,t}(\mathbf{x})\diff\tilde{p}_t(\vo)\diff\vo\Big\vert_{t=0} \\
    &\quad = \int ys_{3,0}(\vx) p(y|\tilde{\va}, \vx)\tilde{p}(\tilde{\va}|\vx)p(\vx)\diff{\vo} \\
    &\quad =\int \left\{\int \mu(\vx, \tilde{\va}) \tilde{p}(\tilde{\va}|\vx)\diff{\tilde{\va}}\right\} s_{3,0}(\vx) p(\vx)\diff{\vo} \\
    &\quad =^{(b)} \int \left\{\int \mu(\vx, \tilde{\va}) \tilde{p}(\tilde{\va}|\vx)\diff{\tilde{\va}} - \vpsi\right\} \\
    &\quad\hspace{2cm} \times \left\{ s_{1,0}(y|\tilde{\va}, \vx) + s_{2,0}(\tilde{\va}|\vx) + s_{3,0}(\vx) \right\} \tilde{p}(\vo)\diff{\vo} \\
    &\quad = \int \left\{\int \mu(\vx, \tilde{\va}) \tilde{p}(\tilde{\va}|\vx)\diff{\tilde{\va}} - \vpsi\right\}  s_0(\vo) \tilde{p}(\vo)\diff{\vo} \\
\end{aligned}
$$
Combining these three terms above proves the condition in (\ref{eq:change}) holds, thus completing the proof.
\end{proof}

\section{Another Example of SRF vs ERF}

The following examples show a simple case where the SRF and ERF estimands are different, thereby demonstrating why the ERF is not a useful estimate of the effect of an exposure shift. Consider a setting in which $\mu(\mX, A) = AX$ with $X \sim N(0,1)$, $A \sim N(X,1)$. Now consider the exposure shift induced by $\tilde{A} = cA$ for some $c\in\R$. Using the SRF formulation, we find that $\vpsi = \E[\mu(\mX, \tilde{A})] = \E[\E[\mX(cA)|\mX]] = c\E[\mX^2] = c$. On the other hand, the E{R}F is $\xi(a) = \E[\mu(\mX, A)|A = a] = a\E[\mX] = 0$ for all $a \in \R$. Therefore, estimators of the two estimands return two different estimates. Moreover, the ERF is identically zero for every treatment value. Thus, it cannot be used to approximate the value of the effect of the exposure shift, even when $\vpsi$ is correctly specified.

\section{Hardware/Software/Data Access}

We ran all of our experiments both in the simulation study and the application section using Pytorch \citep{paszke2019pytorch} on a high-performance computing cluster equipped with Intel 8268 "Cascade Lake" processors. Due to the relatively small size of the datasets, hardware limitations, and the large number of simulations required, we did not require the use of GPUs. Instead we found that using only CPUs run in parallel sufficed. Reproducing the full set of experiments takes approximately 12 hours with 100 parallel processes, each with 4 CPU cores and 8GB of RAM. Each process runs a different random seed for the experiment configuration. 

The code for reproducibility is provided on the submission repository along with the data sources for the simulation experiments. The datasets for these experiments were obtained from the public domain and were adapted from the GitHub repositories shared by \citet{nie2021vcnet} and \citet{bica2020estimating} as explained in the experiment details section. The data for the application was purchased from \url{https://resdac.org/}. Due to a data usage agreement and privacy concerns, manipulation of these data requires IRB approval under which the authors have completed the training and for which reason the data cannot be shared with the public.

\end{document}

%% file: math_commands.tex
\usepackage{amsmath,amsfonts,bm}

\def\eqref#1{equation~\ref{#1}}

\def\1{\bm{1}}

\def\vzero{{\bm{0}}}
\def\vone{{\bm{1}}}
\def\vmu{{\bm{\mu}}}

\def\va{{\bm{a}}}

\def\vo{{\bm{o}}}

\def\vw{{\bm{w}}}
\def\vx{{\bm{x}}}

\def\mA{{\bm{A}}}

\def\mO{{\bm{O}}}

\def\mX{{\bm{X}}}

\def\mZ{{\bm{Z}}}

\DeclareMathAlphabet{\mathsfit}{\encodingdefault}{\sfdefault}{m}{sl}
\SetMathAlphabet{\mathsfit}{bold}{\encodingdefault}{\sfdefault}{bx}{n}

\def\gA{{\mathcal{A}}}

\def\gD{{\mathcal{D}}}

\def\gF{{\mathcal{F}}}

\def\gL{{\mathcal{L}}}
\def\gM{{\mathcal{M}}}

\def\gP{{\mathcal{P}}}

\def\gR{{\mathcal{R}}}

\def\gW{{\mathcal{W}}}
\def\gX{{\mathcal{X}}}
\def\gY{{\mathcal{Y}}}

\def\sD{{\mathbb{D}}}

\def\sP{{\mathbb{P}}}

\newcommand{\E}{\mathbb{E}}

\newcommand{\R}{\mathbb{R}}

\DeclareMathOperator*{\argmin}{arg\,min}